\documentclass[a4paper, 10pt, conference]{ieeeconf}      % Use this line for a4 paper

\IEEEoverridecommandlockouts                              % This command is only needed if 
\usepackage{graphics} % for pdf, bitmapped graphics files
\usepackage{epsfig} % for postscript graphics files
\usepackage{mathptmx} % assumes new font selection scheme installed
\usepackage{times} % assumes new font selection scheme installed
\usepackage{amsmath} % assumes amsmath package installed
\usepackage{amssymb}  % assumes amsmath package installed
\usepackage{graphicx}
\usepackage{graphics,graphicx,amssymb,amsmath,verbatim}
\usepackage{mathrsfs}
\usepackage{amsfonts}
\usepackage{epstopdf}
\usepackage{xcolor}
\usepackage{subfigure}
\usepackage{hyperref}

\newtheorem{theorem}{Theorem}

\newenvironment{proof of Proposition }[1][Proof of Proposition]{\noindent\textit{#1.} }{\ \rule{0.5em}{0.5em}}
\newenvironment{proof of Theorem 2}[1][Proof of Theorem 2]{\noindent\textit{#1.} }{\ \rule{0.5em}{0.5em}}
\newenvironment{proof of Theorem 3}[1][Proof of Theorem 3]{\noindent\textit{#1.} }{\ \rule{0.5em}{0.5em}}
\newenvironment{proof of Theorem 4}[1][Proof of
Theorem 4]{\noindent\textit{#1.} }{\ \rule{0.5em}{0.5em}}
\newenvironment{proof of Theorem 5}[1][Proof of Theorem 5]{\noindent\textit{#1.} }{\ \rule{0.5em}{0.5em}}

\title{\LARGE \bf
Observability Analysis of Graph SLAM-Based Joint Calibration of Multiple Microphone Arrays and Sound Source Localization}

	\author{Yuanzheng He, Jiang Wang, Daobilige Su, Kazuhiro Nakadai, Junfeng Wu, Shoudong Huang, \\
	Youfu Li, and He Kong% <-this %
		\thanks{This paper is accepted to and going to be presented at 2023 IEEE/SICE International Symposium on System Integrations, Atlanta, USA. Corresponding author: H. Kong. Y. He and J. Wang contributed equally to this work. 
		Y. He, J. Wang, and H. Kong are with the Shenzhen Key Laboratory of Biomimetic Robotics and Intelligent Systems, Department of Mechanical and Energy Engineering, Southern University of Science and Technology (SUSTech), Shenzhen, 518055,
        China; they are also affiliated with the Guangdong Provincial Key Laboratory of Human-Augmentation and Rehabilitation Robotics in Universities, SUSTech, Shenzhen, 518055, China (e-mail: 12132259@mail.sustech.edu.cn; 12132297@mail.sustech.edu.cn; kongh@sustech.edu.cn). D. Su is with College of Engineering, China Agricultural University, Beijing, China (email: sudao@cau.edu.cn). K. Nakadai is with the Department of Systems and Control Engineering, Tokyo Institute of Technology, Tokyo, Japan (email: nakadai@ra.sc.e.titech.ac.jp). J. Wu is with the School of Data Science, The Chinese University of Hong Kong, Shenzhen, Shenzhen, P. R. China (email: junfengwu@cuhk.edu.cn). S. Huang is with  the Robotics Institute, University of Technology Sydney, Sydney, Australia (email: shoudong.huang@uts.edu.au). Y. Li is with the Department of Mechanical Engineering, City University of Hong Kong, Hong Kong SAR, China (email: meyfli@cityu.edu.hk).
		}
	}

\begin{document}

\maketitle
\thispagestyle{empty}
\pagestyle{empty}

%%%%%%%%%%%%%%%%%%%%%%%%%%%%%%%%%%%%%%%%%%%%%%%%%%%%%%%%%%%%%%%%%%%%%%%%%%%%%%%%
\begin{abstract}

Multiple microphone arrays have many applications in robot audition, including sound source localization, audio scene perception and analysis, etc. However, accurate calibration of multiple microphone arrays remains a challenge because there are many unknown parameters to be identified, including the Euler angles, geometry, asynchronous factors between the microphone arrays. This paper is concerned with joint calibration of multiple microphone arrays and sound source localization using graph simultaneous localization and mapping (SLAM). By using a Fisher information matrix (FIM) approach, we focus on the observability
analysis of the graph SLAM framework for the above-mentioned calibration problem. We thoroughly investigate the identifiability of the unknown parameters, including the Euler angles, geometry, asynchronous effects between the microphone arrays, and the sound source locations. We establish necessary/sufficient conditions under which the FIM and the Jacobian matrix have full column rank, which implies the identifiability of the unknown parameters. These conditions are closely related to the variation in the motion of the sound source and the configuration of microphone arrays, and have intuitive and physical interpretations. We also discover several scenarios where the unknown parameters are not uniquely identifiable. All theoretical findings are demonstrated using simulation data.

\end{abstract}

\section{INTRODUCTION}

Microphone array-based robot audition systems can be used for a range of applications, such as sound source localization, active multi-mode perception, speech separation, and recognition of multiple sound sources \cite{Grondin2019}-\cite{Rascon C2017}. 
However, accurate calibration of microphone array-based robotic auditory sensors, as for other sensing modalities  such as camera and LIDAR \cite{LIU2022}-\cite{Jiao J2019}, is crucial for satisfactory performance. Hence, calibration of microphone array-based robot audition systems have received much attention in the recent literature. 

For example, a calibration technique was proposed in \cite{Perrodin 2012}, which
allowed estimating microphone position, source position and time offset independent of the calibration signal. Some researchers have tried to use frameworks combining SLAM and beamforming algorithms to perform online calibration of asynchronous
microphones without many measurements of transfer
functions \cite{Nakadai2011a}-\cite{Nakadai2012}. For microphone arrays with asynchronous effects (i.e., clock difference and initial time offset), a systematic examination and observability analysis of SLAM-based microphone array calibration and sound source localization was presented in \cite{Su2015}-\cite{Kong2021} via a FIM approach. However, the above-mentioned methods are only applicable for calibrating a single microphone array. 

Methods for estimating the parameters of multiple microphone arrays have been presented in \cite{Plinge2014}-\cite{Wozniak 2019}. Nevertheless, these methods assumed that the hardware synchronization or orientations of the microphone arrays were known, and only considered scenarios in 2-dimensions (2D). For calibrating multiple microphone arrays, it is necessary to consider not only the geometry and asynchronous effects among the arrays, but also the orientations of  microphone arrays. 

Simultaneous calibration of positions, orientations, time offsets among multiple microphone arrays and sound source location was explored in \cite{Nakadai2021}. In the former work, a combined cost function has been proposed that can allows for estimating the array position, orientation, and time offset concurrently, by using direction of arrival (DOA) information and the time difference of arrival (TDOA) measurements among microphone arrays. However, a thorough analysis regarding the parameter observability in the joint calibration of multiple microphone arrays and sound source localization is still lacking. 

In this study, we will use graph SLAM as a general framework for the above identification question, and concentrate on the parameter identifiablity of the corresponding SLAM problem. By using a FIM approach, we thoroughly investigate the identifiability of the unknown parameters, including the Euler angles,
geometry, asynchronous effects between the microphone arrays,
and the sound source locations. We establish necessary/sufficient
conditions under which the FIM and the Jacobian matrix
have full column rank, which implies the identifiability of
the unknown parameters. These conditions are closely related
to the variation in the motion of the sound source and the
configuration of microphone arrays, and have intuitive and
physical interpretations. We also discover several scenarios
where the unknown parameters are not uniquely identifiable.
All theoretical findings have been validated using simulation data. For readability, most proofs of the theoretical results are put in the Appendix.

%The paper is organized as follows. In Section II, we introduce some preliminaries and problem statement regarding the graph-based SLAM framework for joint calibration of multiple microphone arrays and sound source localization. Section III contains our major results. We present necessary/sufficient conditions for guaranteeing parameter identifiability. We also discover some special cases where observability is impossible and provide a geometric interpretation of these conditions. In Section IV, simulation results are presented to verify the theoretical finds. Section V concludes the paper.

\textbf{Notation}: Denote $x$, $\mathbf{x}$, and $\mathbf{X}$ as
scalars, vectors, and matrices, respectively. $\mathbf{X}^{\mathrm{T}}$ represents the transpose of matrix $\mathbf{X}$. $\mathbf{I}_{n}$ stands for the identity matrix of $n$ dimensions. $\mathbb{R}^{n}$ denotes the $n$-dimensional
Euclidean space. $[a_{1};\cdots;a_{n}]$ denotes $[a_{1}^{\mathrm{T}},\cdots,a_{n}^{\mathrm{T}}]^{\mathrm{T}}$,
where $a_{1},\cdots,a_{n}$ are scalars/vectors/matrices with proper
dimensions. $diag_{n}(\mathbf{A})$ denotes a block diagonal matrix with $\mathbf{A}$ as block diagonal entries for $n$ times; $diag(\mathbf{A},\mathbf{B})$
denotes a block diagonal matrix with $\mathbf{A}$ and $\mathbf{B}$ as
its block diagonal entries; and $\mathbf{0}_{a\times b}$ as a matrix
of dimension $a\times b$ with its all entries as 0. $\mathbf{X}>0$
means that $\mathbf{X}$ is a positive definite matrix. We denote $\left\Vert \mathbf{x}\right\Vert _{\mathbf{P}}^{2}=\mathbf{x}^{%
\mathrm{T}}\mathbf{Px}$. Vectors/matrices, with dimensions not explicitly stated, are assumed to be algebraically compatible.
\section{PRELIMINARIES AND PROBLEM STATEMENT}

\subsection{Graph SLAM for Multiple Microphone Arrays Calibration}

In a calibration scene containing $N$ distributed microphone arrays, the microphone arrays capture $K$ consecutive acoustic signals emitted by a single acoustic source at several spatial positions. As shown in Fig. \ref{TDOA} (here we take $N$=3 as an example), in this paper, simultaneous sound source localization and multiple microphone arrays calibration are performed in a graph-based SLAM framework with three microphone arrays and one moving source.

\begin{figure}[tb]
\centering {\includegraphics[width=0.85\columnwidth]{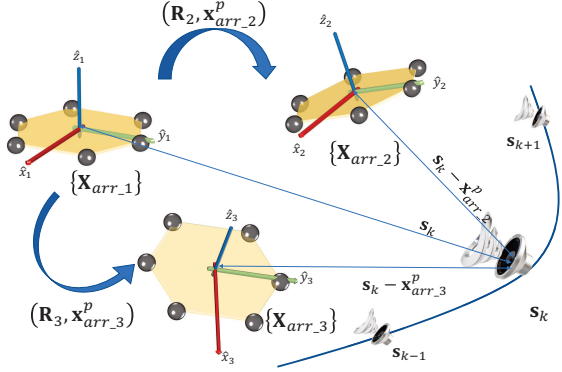}}
\caption{Geometry of the problem setup and graph-based SLAM framework}
\label{TDOA}
\end{figure}

In Fig. \ref{TDOA}, $\mathbf{x}_{arr\_i}^{p}$ represents the location
of the $i\raisebox{0mm}{-}th$ microphone array in the global reference frame, and
any two of arrays are in different positions. We assume that there is a local reference frame $\left\{ \mathrm{\mathbf{x}_{\mathit{arr\_i}}}\right\} $ attached to every microphone array; we choose $\left\{ \mathrm{\mathbf{x}_{\mathit{arr\_\mathrm{1}}}}\right\} $ as the global reference frame; $\mathbf{R}_{\mathit{\mathrm{\mathit{i}}}}$
is the rotation matrix of reference frame $\left\{ \mathrm{\mathbf{x}}_{arr\_1}\right\}$
to the frame $\left\{\mathrm{\mathbf{x}}_{arr\_i}\right\} $
with the rotation angle vector $\mathbf{x}_{arr\_i}^{\theta}$; $\mathbf{s}^{k}$
is the sound source position at time $t^{k},$ $k=1,\ldots,K$, with respect to (w.r.t.)
$\left\{ \mathrm{\mathbf{x}_{\mathit{arr\_\mathrm{1}}}}\right\} $,
where $K$ is the total number of time steps; $d_{i}^{k}$ is the
distance between the $i\raisebox{0mm}{-}th$ \textsuperscript{}microphone array
and the sound source at time instance $t^{k}$. Note that in the calibration process, the multiple microphone
arrays remain static while the sound source moves around in the environment.

Here we consider the most general scenario where there are starting time offset and clock drift among different microphone arrays (we assume that the configuration of each microphone array, including its geometry, is known). When the sound source sends the $k\raisebox{0mm}{-}th$ acoustic signal, the DOA information, i.e., the direction vector of sound source relative to the $i\raisebox{0mm}{-}th$ microphone array frame $\left\{ \mathrm{\mathbf{x}}_{arr\_i}\right\} $
is obtained as follows: 
\begin{equation}
\mathbf{d}_{i1}^{k}=\mathbf{R_{\mathit{\mathrm{\mathit{i}}}}^{\mathrm{\mathit{\mathrm{T}}}}}\frac{\mathbf{s}^{k}-\mathbf{x_{\mathit{arr\_i}}^{\mathit{p}}}}{d_{i}^{k}}.\label{expression_DOA}
\end{equation}

Denote $d_{i}^{k}$, for $i=1,\ldots{,}N$, as the distance between the $i\raisebox{0mm}{-}th$ microphone array and the sound source at the $k\raisebox{0mm}{-}th$ sampling instant. The TDOA information between the  $i\raisebox{0mm}{-}th$ \textsuperscript{}and the first microphone arrays can be expressed as follows: 
\begin{equation}
T_{i1}^{k}=\frac{d_{i}^{k}}{c}-\frac{d_{1}^{k}}{c}+x_{arr\_i}^{\tau}+k{\Delta}_{t}x_{arr\_i}^{\delta}\label{eq:TDOA}
\end{equation}
for $i=2,\ldots{,}N$, where $c$ represents the sound speed in the
air; the scalar (unknown) constant variables $x_{arr\_i}^{\tau}$
and $x_{arr\_i}^{\delta}$ represent the starting time offset and
the clock difference per second of each microphone array, respectively;
${\Delta}_{t}$ is the time interval between two consecutive sound
signals. As mentioned above, the first microphone array is used as the reference, hence 
\begin{equation}\nonumber
\mathbf{x}_{arr\_1}^{p}=\mathbf{0},\text{ } \mathbf{x}_{arr\_1}^{\theta}=\mathbf{0}, \text{ }x_{arr\_1}^{\tau}=0,\text{ }x_{arr\_1}^{\delta}=0.
\label{reference}
\end{equation}
The Euler angles, starting time offsets, and clock differences of the microphone arrays will be determined along with the source positions in the calibration process.

The location and the rotation angle vector of  $i\raisebox{0mm}{-}th$ microphone array (where $i=2,\ldots{,}N$), i.e., $\mathbf{x}_{arr\_i}^{p}$ and $\mathbf{x}_{arr\_i}^{\theta}$, can be expressed as: 
\begin{equation}\nonumber
\begin{array}{c}
\mathbf{x}_{arr\_i}^{p}=\left[x_{arr\_i}^{x};x_{arr\_i}^{y};x_{arr\_i}^{z}\right], \text{ } \mathbf{x}_{arr\_i}^{\theta}=\left[\theta_{arr\_i}^{x};\theta_{arr\_i}^{y};\theta_{arr\_i}^{z}\right],
\end{array}
\label{eq:X}
\end{equation}
respectively, where $\theta_{arr\_i}^{x},\theta_{arr\_i}^{y}$, and $\theta_{arr\_i}^{z}$
take values in the range of $[0,2\pi],[0,\pi]$, and $[0,2\pi],$ respectively.
Denote the unknown parameters w.r.t. the  $i\raisebox{0mm}{-}th$ microphone array as:
\begin{equation}\nonumber
\mathbf{x}_{arr\_i}=\left[\mathbf{x}_{arr\_i}^{p};\mathbf{x}_{arr\_i}^{\theta};x_{arr\_i}^{\tau};x_{arr\_i}^{\delta}\right].\label{eq_mic_state_vector_3d}
\end{equation}
Hence, all the unknown parameters w.r.t. the microphone arrays are: 
\begin{equation}\nonumber
\mathbf{x}_{arr}=\left[\mathbf{x}_{arr\_2};\cdots;\mathbf{x}_{arr\_N}\right].\label{eq_mic_state_vector}
\end{equation}
Denote the sound source position at time $t^{k},$ $k=1,\ldots,K$
as:
\begin{equation}\nonumber
\mathbf{s}^{k}=\left[s_{x}^{k};s_{y}^{k};s_{z}^{k}\right]. \label{eq:Source location}
\end{equation}
Thus, all unknown parameters to be identified are: 
\begin{equation}\nonumber
\mathbf{x}=\left[\mathbf{x}_{arr};\mathbf{s}^{1};\cdots;\mathbf{s}^{K}\right].\label{eq_state_vector}
\end{equation}
We denote the ideal TDOA and DOA measurement information at the  $k\raisebox{0mm}{-}th$\textsuperscript{} time instance as: 
\begin{equation}
\mathbf{z}^{k}=\left[T_{21}^{k};\mathbf{d}_{21}^{k};T_{31}^{k};\mathbf{d}_{31}^{k};\cdots;T_{N1}^{k};\mathbf{d}_{N1}^{k}\right]\in\mathbf{\mathbb{R}}^{4(N-1)}.\label{eq_p_l_ob}
\end{equation}
The real values of DOA and TDOA measurements at time $k$ are also subject to the influence of Gaussian noise as follows:
\begin{equation}
\mathbf{y}^{k}=\mathbf{z}^{k}+\mathbf{v}^{k}\label{measurements}
\end{equation}
where $\mathbf{z}^{k}$ is defined in (\ref{eq_p_l_ob}), $\mathbf{v}^{k}\sim\mathcal{N}(0,\mathbf{P})$,
with $\mathbf{\mathbb{\mathbf{P}}}>0\in\mathbf{\mathbb{R}}^{4(N-1)\times4(N-1)}$.
We assume that the sound source relative position between two consecutive time steps can be measured with Gaussian noise, i.e.,
\begin{equation}
\mathbf{s}_{\Delta}^{k}=\mathbf{s}^{k+1}-\mathbf{s}^{k}+\mathbf{w}^{k}\label{random_walk}
\end{equation}
where $k=1,...,K-1$, $\mathbf{w}^{k}\sim\mathcal{N}(0,\mathbf{Q})$,
with $\mathbf{Q}>0\in\mathbf{\mathbb{R}}^{3\times3}$. 
We combine the relative position measurements, the TDOA, and DOA measurements as:
\begin{equation}\nonumber
\mathbf{m}=\left[\mathbf{y}_{1};\mathbf{s}_{1}^{\Delta};\mathbf{y}_{2};\mathbf{s}_{2}^{\Delta};\cdots;\mathbf{y}_{K-1};\mathbf{s}_{K-1}^{\Delta};\mathbf{y}_{K}\right]\label{eq_ob_M-1}
\end{equation}
where $\mathbf{s}_{k}^{\Delta}$ and $\mathbf{y}_{k}$ are defined
in (\ref{random_walk}) and (\ref{measurements}), respectively. Then the models in (\ref{measurements})-(\ref{random_walk}) can
be rewritten in a compact form as:
\begin{equation}
\mathbf{m}=\mathbf{g}(\mathbf{x})+\mathbf{\gamma}\label{eq_ob_M_2-1}
\end{equation}
where $\mathbf{g}(\mathbf{x})$ is the combined observation model,
and $\mathbf{\gamma}\sim\mathcal{N}(0,\mathbf{W})$ is the noise of
combined observations with 
\begin{equation}
\mathbf{W}=diag(diag_{K-1}(\mathbf{P,Q),P}).\label{eq_ob_W-1}
\end{equation}
Based on the above discussions, our focus is
to identify the parameters of multiple microphone arrays (microphone
arrays positions, orientations, time offsets, and clock offsets) and
sound source positions. As shown in Fig. \ref{TDOA}, the graph-based
SLAM framework is a feasible solution to the above problems by treating the moving sound source as a robot and the multiple microphone arrays as landmarks \cite{Grisetti2010}. As in \cite{Su2015}-\cite{Su2020}, the parameter identification problem for asynchronous multiple microphone arrays can be treated as the following standard least squares (LS) problem using graph SLAM:
\begin{equation}
\arg \min\limits_{\widehat{\mathbf{x}}}\left\Vert \mathbf{m-}g(\widehat{%
\mathbf{x}})\right\Vert _{\mathbf{W}^{-1}}^{2} \label{eq_ls}
\end{equation}
where $\widehat{\mathbf{x}}$ represents the estimate of all the unknown parameters. The measurements obtained by different microphone arrays constitute the spatial constraints and can be included in the above LS to improve estimation accuracy (due to limited space, we will not elaborate on these details in the remainder of the paper).

\subsection{The Corresponding FIM and Problem Statement}

By using the FIM approach, we know that the observability of the graph-based SLAM problem described earlier depends on whether the FIM is non-singular \cite{Kong2021}. 
For non-random vector parameter estimation, the FIM of an unbiased
estimator is defined as 
\begin{equation}
\mathbf{I}_{FIM}\triangleq E\{[\nabla_{\mathbf{x}}ln\Lambda(\mathbf{x})][\nabla_{\mathbf{x}}ln{\Lambda}(\mathbf{x})]^{\mathrm{T}}\}\label{eq_ob_J-1}
\end{equation}
where $\Lambda(\mathbf{x})\triangleq p(\mathbf{m}|\mathbf{x})$ is
the likelihood function, and the partial derivatives should be calculated at the true value of $\mathbf{x}$ \cite[chap. 2]{Bar-Shalom2004}.
By following similar arguments as those in \cite{Dissanayake2008} and
\cite{Huang2016}, the FIM in (\ref{eq_ob_J-1}) for models in (\ref{eq_ob_M_2-1})-(\ref{eq_ob_W-1})
can be formulated as 
\begin{equation}
\mathbf{I}_{FIM}\mathbf{=J}^{\mathrm{T}}\mathbf{W}^{-1}\mathbf{J}\label{J-1}
\end{equation}
where $\mathbf{J}$ is the Jacobian matrix of the function $\mathbf{g}(\bullet)$
in (\ref{eq_ob_M_2-1}) w.r.t. $\mathbf{x}$ \cite[pp. 569]{Siciliano2009},
and its explicit expression will be given later in the paper (see
in (\ref{eq:Jacobi})). When $\mathbf{W}>0,$ one has that 
\begin{equation}\nonumber
rank(\mathbf{J})=rank(\mathbf{I}_{FIM}).\label{eq_ob_rank_J-1}
\end{equation}
The question of interests is formally stated as follows.
%\begin{problem}

\textit{Problem:}
\label{problem1-1}Given the problem setup described as above, find
conditions under which the FIM $\mathbf{I}_{FIM}$ defined in (\ref{eq_ob_J-1})
is non-singular, or equivalently, the Jacobian matrix $\mathbf{J}$ is of full column rank. 
%\end{problem}

\section{MAIN RESULTS}
By leveraging the structure of the Jacobian matrix $\mathbf{J}$ associated with the SLAM formulation, we next establish necessary/sufficient conditions for the non-singularity of the $\mathbf{I}_{FIM}$ and the observability of the SLAM problem. In addition, we will reveal some special cases when the Jacobian matrix or FIM cannot have full column rank.

\subsection{Main Results}
From the definition of the Jacobian matrix \cite[pp. 569]{Siciliano2009}, we know that $\mathbf{J}\in\mathbb{R}^{g_{1}\times g_{2}}$, $g_{1}=4(N-1)K+3(K-1)$,
$g_{2}=8(N-1)+3K$. From (\ref{eq_ob_J-1})-(\ref{J-1}), a necessary
and sufficient condition for $\mathbf{I}_{FIM}$ to be nonsingular
is that $\mathbf{J}$ has full column rank. For $\mathbf{J}$ to be of full column rank, it is necessary that 
\begin{equation}\nonumber
\begin{array}{l}
4(N-1)K+3(K-1)\geq8(N-1)+3K\\
\implies K\geqslant\left\lceil 2+\dfrac{3}{4(N-1)}\right\rceil,
\end{array}\label{eq:neccessary_condion}
\end{equation}
where $%
\left\lceil \cdot \right\rceil $ stands for the ceiling operation generating
the least integer not less than the number within the operator. We then have the following results.

%\begin{proposition} 
\textit{Proposition:} The Jacobian $\mathbf{J}$ can be written as
\begin{equation}
\mathbf{J}=\underset{\mathbf{J}_{1}}{\underbrace{\left[\begin{array}{c}
\mathbf{L}^{1}\\
\mathbf{0}_{3\times8\widetilde{N}}\\
\mathbf{L}^{2}\\
\mathbf{0}\\
\vdots\\
\mathbf{L}^{K-1}\\
\mathbf{0}\\
\mathbf{L}^{K}
\end{array}\right.}}\underset{\mathbf{J}_{2}=\left[\begin{array}{llll}
\mathbf{J}_{2}^{1} & \mathbf{J}_{2}^{2}\cdots & \mathbf{J}_{2}^{K-1} & \mathbf{J}_{2}^{K}\end{array}\right]}{\underbrace{\left.\begin{array}{ccccc}
\mathbf{T}^{1} & \mathbf{0}_{4\widetilde{N}\times3} & \cdots & \mathbf{0} & \mathbf{0}\\
-\mathbf{I}_{3} & \mathbf{I}_{3} & \cdots & \mathbf{0} & \mathbf{0}\\
\mathbf{0}_{4\widetilde{N}\times3} & \mathbf{T}^{2} & \cdots & \mathbf{0} & \mathbf{0}\\
\mathbf{0} & -\mathbf{I}_{3} & \cdots & \mathbf{0} & \mathbf{0}\\
\vdots & \vdots & \ddots & \vdots & \vdots\\
\mathbf{0} & \mathbf{0} & \cdots & \mathbf{T}^{K-1} & \mathbf{0}\\
\mathbf{0} & \mathbf{0} & \cdots & -\mathbf{I}_{3} & \mathbf{I}_{3}\\
\mathbf{0} & \mathbf{0} & \cdots & \mathbf{0} & \mathbf{T}^{K}
\end{array}\right]}}\label{eq:Jacobi}
\end{equation}
where $\widetilde{N}=N-1$, expressions of $\mathbf{L}^{k}$, $\mathbf{T}^{k}$,
for $k=1,...,K$, can be found in (\ref{eq:L}) and (\ref{eq:part TK}).
%\end{proposition}

%\begin{proof}
%See Appendix.\end{proof}

\begin{theorem} The Jacobian matrix $\mathbf{J}$ is of full column
rank if and only if the following matrix 
\begin{equation}\nonumber
\mathbf{F}=\underset{\mathbf{L}}{\underbrace{\left[\begin{array}{c}
\mathbf{L}^{1}\\
\mathbf{L}^{2}\\
\vdots\\
\mathbf{L}^{K}
\end{array}\right.}}\underset{\mathbf{T}}{\underbrace{\left.\begin{array}{c}
\mathbf{T}^{1}\\
\mathbf{T}^{2}\\
\vdots\\
\mathbf{T}^{K}
\end{array}\right]}}\label{eq:L,T}
\end{equation}
is of full column rank. \end{theorem}
\begin{proof}
The proof follows similarly from \cite{Kong2021}
and is skipped here.
\end{proof}

\begin{theorem} The Jacobian matrix $\mathbf{J}$ is of full column
rank only if the matrix $\mathbf{\bar{T}}$ and $\mathbf{\mathbf{\mathbf{\bar{L}}_{i}}}$, for $i=2,\ldots,N,$ are of full column rank, respectively, where

\begin{equation}
\bar{\mathbf{T}}=\left[\begin{array}{c}
\mathbf{0_{\mathrm{2\times3}}}\\
{\scriptstyle {\scriptstyle -\dfrac{{\scriptstyle \left({\scriptstyle \mathbf{s}^{1}}\right)^{\mathrm{T}}}}{{\scriptstyle cd_{1}^{1}}}+2\dfrac{{\scriptstyle \left({\scriptstyle \mathbf{s}^{2}}\right)^{\mathrm{T}}}}{{\scriptstyle cd_{1}^{2}}}}-\dfrac{{\scriptstyle \left({\scriptstyle \mathbf{s}^{3}}\right)^{\mathrm{T}}}}{{\scriptstyle cd_{1}^{3}}}}\\
{\scriptstyle {\scriptstyle -2\dfrac{{\scriptstyle \left({\scriptstyle \mathbf{s}^{1}}\right)^{\mathrm{T}}}}{{\scriptstyle cd_{1}^{1}}}}+3\dfrac{{\scriptstyle \left({\scriptstyle \mathbf{s}^{2}}\right)^{\mathrm{T}}}}{{\scriptstyle cd_{1}^{2}}}-\dfrac{{\scriptstyle \left({\scriptstyle \mathbf{s}^{4}}\right)^{\mathrm{T}}}}{{\scriptstyle cd_{1}^{4}}}}\\
\vdots\\
{\scriptstyle {\scriptstyle \left({\scriptstyle -\left(K-2\right)\dfrac{{\scriptstyle \left({\scriptstyle \mathbf{s}^{1}}\right)^{\mathrm{T}}}}{{\scriptstyle cd_{1}^{1}}}}+\left(K-1\right){\scriptstyle \dfrac{{\scriptstyle \left({\scriptstyle \mathbf{s}^{2}}\right)^{\mathrm{T}}}}{{\scriptstyle cd_{1}^{2}}}}-\dfrac{{\scriptstyle \left({\scriptstyle \mathbf{s}^{K}}\right)^{\mathrm{T}}}}{{\scriptstyle cd_{1}^{K}}}\right)}}\\
\mathbf{0_{\mathrm{\mathit{3K}\times3}}}
\end{array}\right]\label{eq:T-BAR}
\end{equation}

and

\begin{equation}
\begin{array}{c}
\mathbf{\mathbf{\bar{L}}_{i}}=\left[\begin{array}{cccc}
1 & 0 & \mathbf{0} & \mathbf{0}\\
0 & 1 & \mathbf{0} & \mathbf{0}\\
0 & 0 & {\scriptstyle \mathbf{h}_{arr\_i}^{1}-2\mathbf{h}_{arr\_i}^{2}+\mathbf{h}_{arr\_i}^{3}} & \mathbf{0}\\
0 & 0 & {\scriptstyle 2\mathbf{h}_{arr\_i}^{1}-3\mathbf{h}_{arr\_i}^{2}+\mathbf{h}_{arr\_i}^{4}} & \mathbf{0}\\
\vdots & \vdots & \vdots & \vdots\\
0 & 0 & {\scriptstyle \left(K-2\right)\mathbf{h}_{arr\_i}^{1}-\left(K-1\right)\mathbf{h}_{arr\_i}^{2}+\mathbf{h}_{arr\_i}^{K}} & \mathbf{0}\\
\mathbf{0} & \mathbf{0} & \mathbf{U}_{arr\_i}^{1} & \mathbf{V}_{arr\_i}^{1}\\
\mathbf{0} & \mathbf{0} & \mathbf{U}_{arr\_i}^{2} & \mathbf{V}_{arr\_i}^{2}\\
\vdots & \vdots & \vdots & \vdots\\
\mathbf{0} & \mathbf{0} & \mathbf{U}_{arr\_i}^{K} & \mathbf{V}_{arr\_i}^{K}
\end{array}\right]\end{array},\label{eq: L-BAR}
\end{equation}
where $\mathbf{h}$, $\mathbf{U}$, and $\mathbf{V},$ can be found in (\ref{eq:Block of neccessary matrix}).

\end{theorem} 
%\begin{proof} See Appendix.\end{proof}
\begin{theorem} The Jacobian matrix $\mathbf{J}$ is of full column
rank if the following statements hold.

(i) Any matrix consisting of the $(j-1)\raisebox{0mm}{-}th$ column block and the
last column block in $\overline{\mathbf{F}}^{\prime}$ is of full
column rank, $2\leq j\leq N$.

(ii) All matrices $\mathbf{\mathbf{\bar{L}}_{i}}$ in (\ref{eq:F bar prime}),
for the multiple microphone arrays system, $i=2,\ldots,N$ and $i\neq j$ are of full column rank.\end{theorem} 
%\begin{proof}
%See Appendix.\end{proof}
\subsection{Special Cases When Observability is Impossible}
Next, we state some exceptional cases when observability is impossible.

\begin{theorem}The matrix $\mathbf{\bar{T}}$ is not of full column rank if one or more of the following conditions hold.

(i) For all microphone arrays, there exists at least five time steps
information (for this to hold, we must have $K\geq5$ in (\ref{eq:T-BAR})), i.e., when $K<5,$ the Jacobian matrix $\mathbf{J}$ is not of full
column rank.

(ii) The coordinates of the sound source at all moments are collinear (together with the origin)
in $\left\{ \mathrm{\mathbf{x}}_{arr\_1}\right\} $, i.e., $\mathbf{\mathbf{s}}^{k}=\mathbf{\mathbf{{\scriptstyle \lambda}s}}^{k-1}$
does always hold, where ${\scriptstyle {\textstyle \lambda}}$ is an arbitrary
real number.

(iii) The sound source keeps moving at any plane of $x=\alpha y$,
$x=\beta z$, and $y=\gamma z$ w.r.t. $\left\{ \mathrm{\mathbf{x}}_{arr\_1}\right\} $
in all moments, where $\alpha,\beta,\gamma$ are arbitrary real
numbers.
\end{theorem}
%\begin{proof}
%See Appendix. \end{proof}

\begin{theorem}The matrix $\mathbf{\mathbf{\bar{L}}_{i}}$, $i=2,3,\cdots,N$, are not of full column rank if one or more of the following conditions hold:

(i) The coordinates of the sound source at all moments
are proportional w.r.t. $\left\{ \mathrm{\mathbf{x}}_{arr\_i}\right\} $,
i.e., $(\mathbf{\mathbf{s}}^{k}-\mathbf{x}_{arr\_i}^{p})=\lambda(\mathbf{\mathbf{s}}^{k-1}-\mathbf{x}_{arr\_i}^{p})$
does always hold, where ${\scriptstyle {\textstyle \lambda}}$ is an
arbitrary real number;

(ii) For the $i\raisebox{0mm}{-}th$ microphone array, one of the Euler
angles satisfies $\theta_{arr\_i}^{y}=\frac{\pi}{2}$. 

\end{theorem}
%\begin{proof}
%See Appendix.\end{proof}

\begin{figure}[tb]
\centering \subfigure[]{\includegraphics[width=0.75\columnwidth]{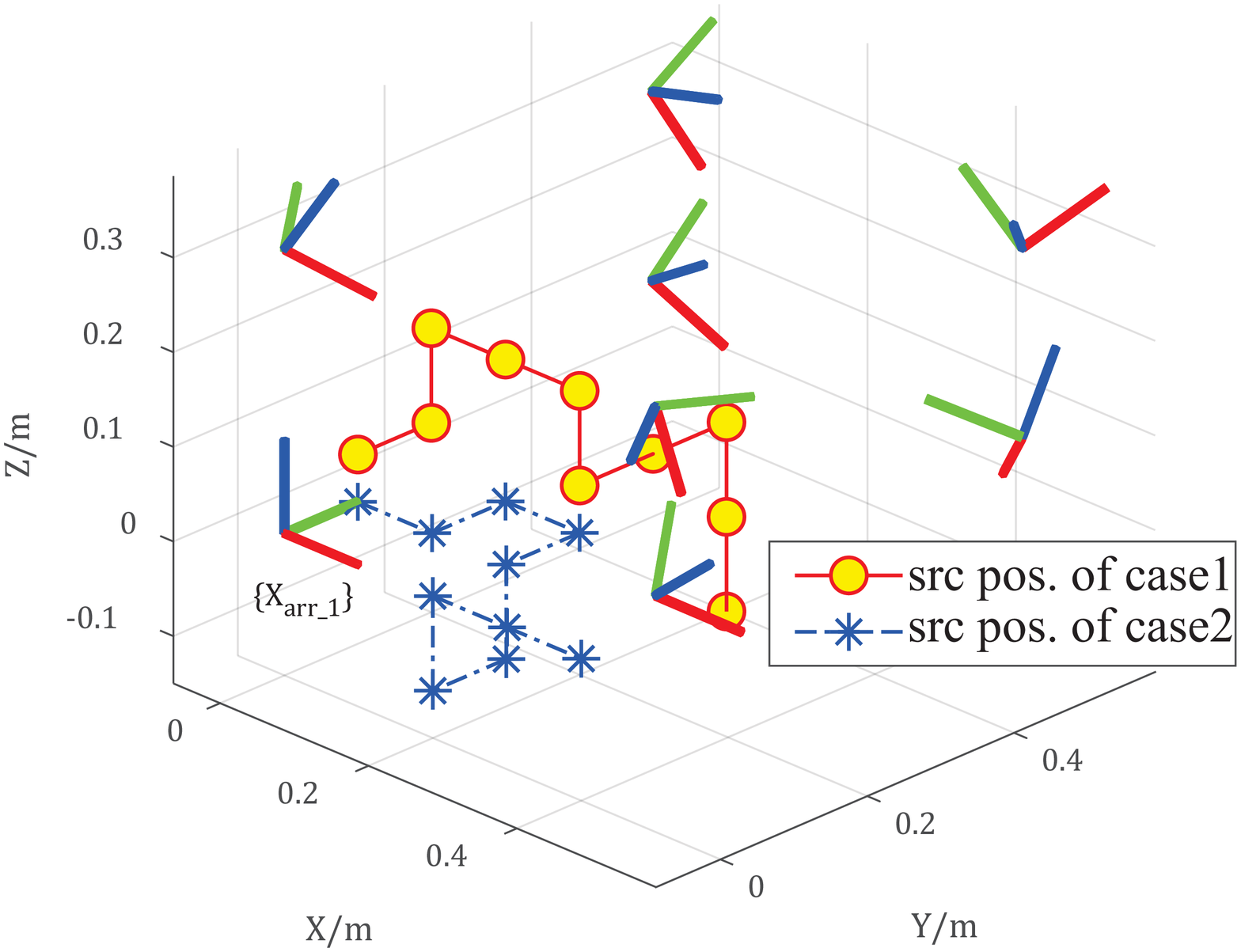}}\\
\subfigure[]{\includegraphics[width=0.75\columnwidth,height=0.5\linewidth]{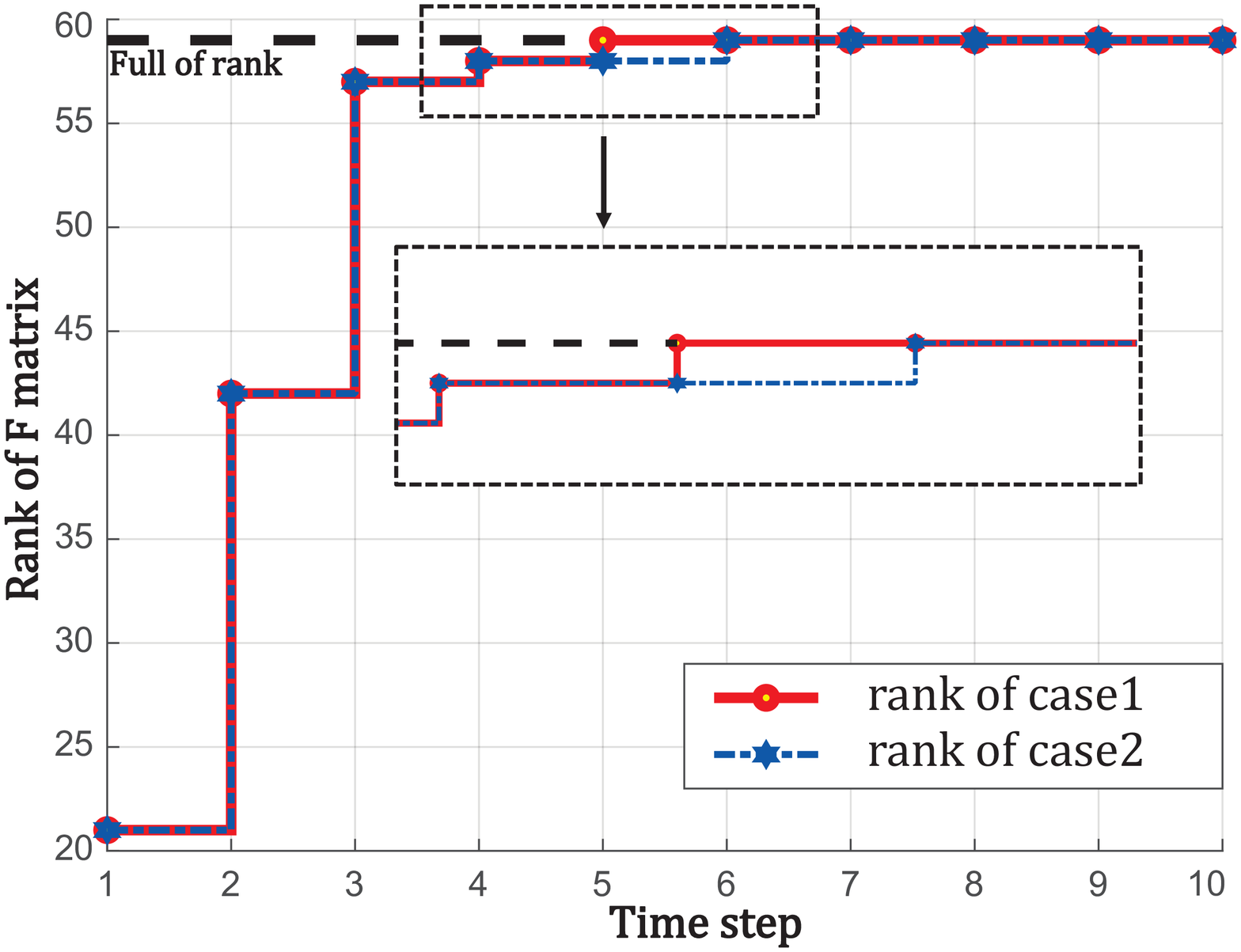}}
\caption{Two observable cases and the variation of the $\mathbf{F}$ matrix's
rank. (a) Geometric relationship between the source and the microphone
arrays during the movement. (b) Variation of the $\mathbf{F}$ matrix rank
with the movement of the source.}
\label{fig_lemma_corrolary-2}
\end{figure}

\begin{figure}[tbh]
\centering \subfigure[]{\includegraphics[width=0.75\columnwidth]{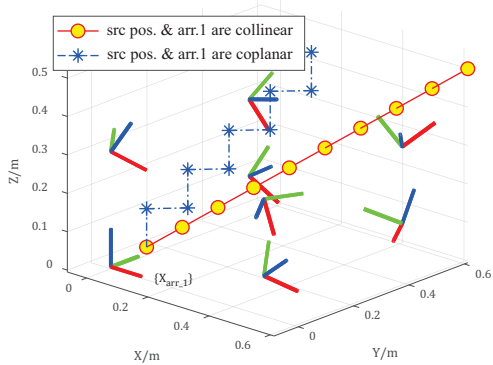}}\\
\subfigure[]{\includegraphics[width=0.75\columnwidth,height=0.5\linewidth]{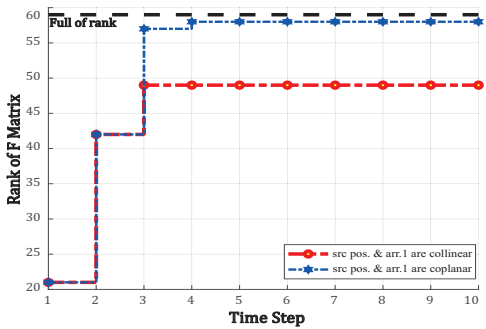}}
\caption{The sound source remains co-linear or co-planar with $\left\{ \mathrm{\mathbf{x}}_{arr\_1}\right\}$
during the movement. (a) Geometric relationship between the source
and the microphone arrays during the movement. (b) Variation of the
$\mathbf{F}$ matrix rank with the movement of the source.}
\label{fig_lemma_corrolary}
\end{figure}

\begin{figure}[tbh]
\centering \subfigure[]{\includegraphics[width=0.75\columnwidth]{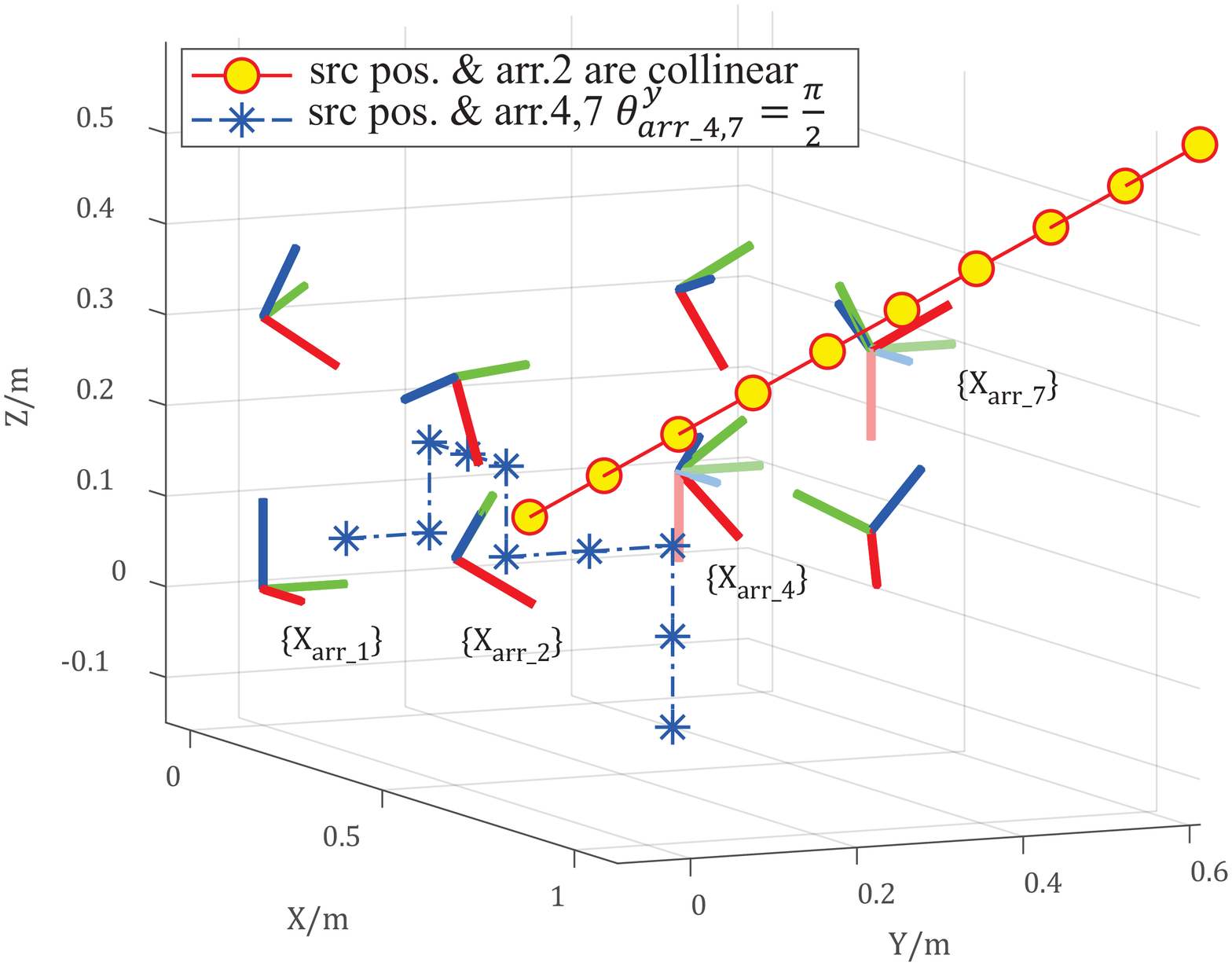}}\\
\subfigure[]{\includegraphics[width=0.7\columnwidth,height=0.5\linewidth]{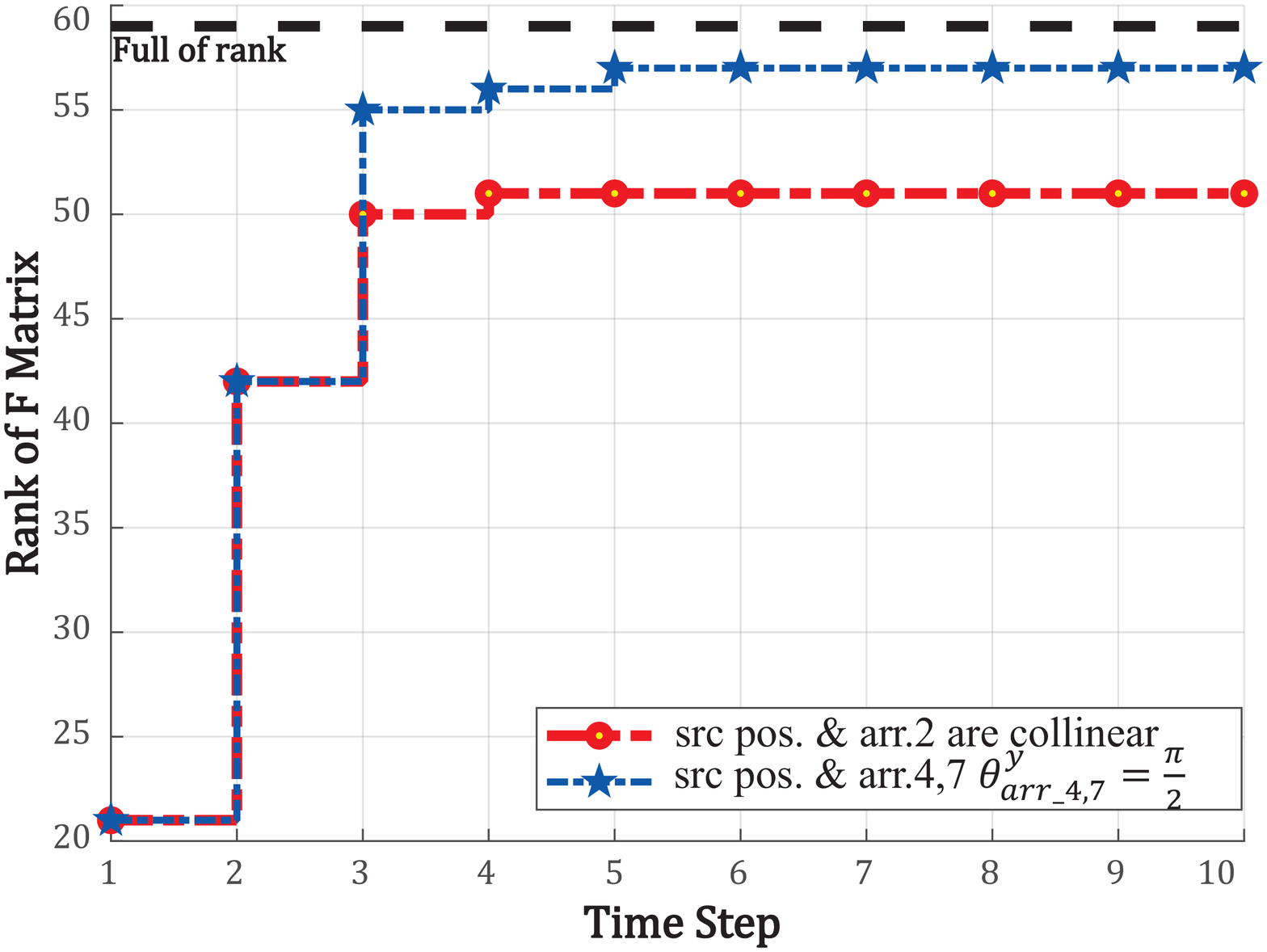}}
\caption{The sound source remains co-linear with $\left\{ \mathrm{\mathbf{x}}_{arr\_2}\right\} $
or $\theta_{arr\_4,7}^{y}=\pi/2$ during the movement. (a) Geometric
relationship between the source and the microphone arrays during the
movement. (b) Variation of the $\mathbf{F}$ matrix rank with the movement
of the source. }
\label{fig_lemma_corrolary-1}
\end{figure}

\section{\label{simu}Numerical Simulations and Results}

We next use numerical simulations to illustrate the theoretical findings obtained above. The whole experimental scheme is shown in Fig. \ref{TDOA}, where the multiple microphone arrays remain static while the sound source moves around in the environment. To generate the data, we assume that the characteristic parameters of each microphone array and sound source positions are known. All TDOA and DOA measurements are corrupted by Gaussian noise. Here we consider the case with eight microphone arrays.

In the simulation process, we set  $\left\{ \mathrm{\mathbf{x}}_{arr\_1}\right\} $ as
the global reference coordinate system. 
The sound source always moves at a speed of 0.1m/s and emits acoustic signal once per second. The starting time offset of each microphone array is randomly generated in 0$\sim$0.1s, and the clock drift constant is randomly generated in 0$\sim$0.1ms to restore the real scene as much as possible.

\subsection{Observable Cases}

We firstly give two observable scenarios for which the motion trajectories of the
sound source in 3D space are shown in Fig. \ref{fig_lemma_corrolary-2}(a). The variation of $\mathbf{F}$ matrix rank with time steps is shown in
Fig. \ref{fig_lemma_corrolary-2}(b). It can be seen that as time steps increase and the sound source moves along the two trajectories, the $\mathbf{F}$ matrix gradually becomes full column rank which indicates that the Jacobian matrix also gradually becomes full column rank. Based on Theorem 3, since $rank(\mathbf{M}_{2\_T})=11$
and $rank(diag(\mathbf{\bar{L}}_{i}))=48$, $i=3,4,\cdots,8$, the Jacobian matrix is of full column rank. At the moment when the Jacobian matrix becomes full column rank, it can be verified that $rank(diag(\mathbf{\bar{L}}_{i}))=56$, $i=2,3,\cdots,8$, and $rank(\mathbf{\bar{T}})=3$. 
Hence, the simulations presented so far based on the theorems worked properly as expected. It is worth noting that the sound source positions are not always in the same plane or same line. Therefore, the Jacobian matrix is of full column rank in general.

\subsection{Unobservable Cases}

Several unobservable scenarios are presented in the following to verify the conclusions in Theorems 4-5.

(i) For the Jacobian matrix to have full column rank, it is necessary that the time steps are greater than or equal to 3 so that the number
of rows of the Jacobian matrix is greater than the number of
columns. As can be seen from Fig. \ref{fig_lemma_corrolary-2}(b),
when the number of time steps is greater than or equal to 3 but less
than 5, the Jacobian matrix is not of full column rank.
This reflects that the system is unobservable when
the number of time steps is less than 5.

(ii) For the trajectories of the sound source shown in Fig. \ref{fig_lemma_corrolary}(a),
the first case is that the sound source stays co-linear with $\left\{ \mathrm{\mathbf{x}}_{arr\_1}\right\} $
during the moving process, and the second case is that the sound source
remains co-planar with $\left\{ \mathrm{\mathbf{x}}_{arr\_1}\right\} $.
From Fig. \ref{fig_lemma_corrolary}(b), it can be seen that both are permanently unobservable due to the lack of information.

(iii) For the sound source trajectories shown in Fig. \ref{fig_lemma_corrolary-1}(a),
the first case is that the sound source keeps co-linear with the origin of $\left\{ \mathrm{\mathbf{x}}_{arr\_2}\right\} $
during the movement. In the second case, the Euler angles $\theta_{arr\_4}^y$ and $\theta_{arr\_7}^y$ of $\left\{ \mathrm{\mathbf{x}}_{arr\_4}\right\} $ and $\left\{ \mathrm{\mathbf{x}}_{arr\_7}\right\} $ are $\frac{\pi}{2}$,
and the sound source travels along the route of the observable
scenario mentioned in case 1 of Fig. \ref{fig_lemma_corrolary-2}(a). The rotation angle is at the singular point of observation, rendering the system unobservable. Hence, the simulations presented above validate the conclusions in Theorems 4-5.

\section{\label{CONCLUSION}CONCLUSION}
This paper is concerned with the observability analysis of graph SLAM-based joint calibration of
multiple microphone arrays and sound source localization. Via a FIM approach, we thoroughly investigate the identifiability of the unknown parameters, including the Euler angles, geometry, asynchronous effects between the microphone arrays, and the sound source locations. We establish necessary/sufficient
conditions under which the FIM and the Jacobian matrix
have full column rank, which implies the identifiability of the unknown parameters. These conditions are closely related to the variation in the motion of the sound source and the configuration of microphone arrays, and have intuitive and physical interpretations. Based on these conditions, we also find some special cases when the Jacobian matrix does not have full column rank, and provide some geometric and physical interpretations. Extensive simulations have been conducted to  demonstrate the theoretical findings. The focus of our current and further work is to develop and validate calibration algorithms for multiple microphone arrays.

\section{Acknowledgment}

This work was supported by the Science, Technology, and Innovation Commission of Shenzhen Municipality [Grant No. ZDSYS20200811143601004].

\begin{appendix}
\begin{proof of Proposition }Firstly, we note that the relative position of the sound source satisfies
\begin{equation}\nonumber
\mathbf{s}_{\Delta}^{k-1}=\mathbf{s}^{k}-\mathbf{s}^{k-1}+\mathbf{w}^{k-1}
\end{equation}
whose corresponding Jacobian matrices are
\begin{equation}\nonumber
\dfrac{\partial\mathbf{s}_{\Delta}^{k-1}}{{\partial}\mathbf{s}^{k-1}}=-\mathbf{I}_{3},\text{ }\dfrac{\partial\mathbf{s}_{\Delta}^{k-1}}{{\partial}\mathbf{s}^{k}}=\mathbf{I}_{3}.
\end{equation}
Secondly, for $i=2,...,N$, the distance between the  $i\raisebox{0mm}{-}th$ microphone
array and the sound source at time instance $t^{k}$ can be computed
as 
\begin{equation}\nonumber
d_{i}^{k}=\sqrt{{({\Delta x}_{i}^{k})}^{2}+{({\Delta y}_{i}^{k})}^{2}+{({\Delta z}_{i}^{k})}^{2}}\label{dist}
\end{equation}
where ${\Delta x}_{i}^{k}=s_{x}^{k}-x_{arr\_i}^{x}, \text{ } {\Delta}y_{i}^{k}=s_{y}^{k}-x_{arr\_i}^{y}, \text{ } {\Delta}z_{i}^{k}=s_{z}^{k}-x_{arr\_i}^{z}$. When $i=1,$ i.e., for the first microphone array, we have 
\begin{equation}
d_{1}^{k}=\sqrt{{(s_{x}^{k})}^{2}+{(s_{y}^{k})}^{2}+{(s_{z}^{k})}^{2}}.\label{eq:s}
\end{equation}

Denote $\mathbf{L}^{k}=\dfrac{\partial\mathbf{z}^{k}}{{\partial}\mathbf{x}_{arr}}$,
i.e., $\mathbf{L}^{k}$ is the derivative of $\mathbf{z}^{k}$ (the
measurements at the time step $t^{k}$) w.r.t. $\mathbf{x}_{arr}$.
Based on the DOA and TDOA information in (\ref{expression_DOA})--(\ref{eq:TDOA}), we then have: 
\begin{equation}
\mathbf{L}^{k}=\dfrac{\partial\mathbf{z}^{k}}{{\partial}\mathbf{x}_{arr}}=\left[\begin{array}{ccc}
\mathbf{J}_{arr\_2}^{k}, \cdots, \mathbf{J}_{arr\_N}^{k}\end{array}\right]\in\mathbb{R}^{4\widetilde{N}\times8\widetilde{N}}\label{eq:L}
\end{equation}
where for $i=2,...,N,$ and $k=1,\ldots,K$, and only entries of $\mathbf{J}_{arr\_i}^{k}$ on its $(4i-7:4i-4)$
rows are nonzero. Denote{\small{} }$\mathbf{h}_{i}^{k},\mathbf{U}_{i}^{k}$
as the partial derivative of TDOA and DOA w.r.t. microphone array
position, respectively; denote $\mathbf{V}_{i}^{k}$ as the partial
derivative of DOA w.r.t. $X,Y,Z$ Euler angles. We then have: 
\begin{equation}
\begin{array}{c}
\mathbf{H}_{arr\_i}^{k}\triangleq\mathbf{J}_{arr\_i}^{k}(4i-7:4i-4,:)\\
=\left[\begin{array}{cccc}
\mathbf{h}_{i}^{k} & \mathbf{0}_{1\times3} & 1 & {k\Delta}_{t}\\
\mathbf{U}_{i}^{k} & \mathbf{V}_{i}^{k} & \mathbf{0}_{3\times1} & \mathbf{0}_{3\times1}
\end{array}\right]\in\mathbf{\mathbb{R}}^{4\times8}
\end{array}\label{eq:Block of neccessary matrix}
\end{equation}
where
\begin{equation}\nonumber
\mathbf{h}_{i}^{k}=\text{\ensuremath{\left[\dfrac{{\scriptstyle {\displaystyle -{\Delta x}_{i}^{k}}}}{cd_{i}^{k}},\dfrac{{\scriptstyle {\displaystyle -{\Delta y}_{i}^{k}}}}{cd_{i}^{k}},\dfrac{{\scriptstyle {\displaystyle -{\Delta z}_{i}^{k}}}}{cd_{i}^{k}}\right]}},
\end{equation}
\begin{equation}
\begin{array}{l}
\mathbf{U}_{i}^{k}=-\mathbf{R}_{i}^{\mathrm{T}}\mathbf{A} \\
=-\mathbf{R}_{i}^{\mathrm{T}}\left[
\begin{array}{ccc}
\dfrac{{\scriptstyle(\Delta y_{i}^{k})^{2}+(\Delta z_{i}^{k})^{2}}}{{%
\scriptstyle(d_{i}^{k})^{3}}} & \dfrac{{\scriptstyle-\Delta x_{i}^{k}\Delta
y_{i}^{k}}}{{\scriptstyle(d_{i}^{k})^{3}}} & \dfrac{{\scriptstyle-\Delta
x_{i}^{k}\Delta z_{i}^{k}}}{{\scriptstyle(d_{i}^{k})^{3}}} \\
\dfrac{{\scriptstyle-\Delta x_{i}^{k}\Delta y_{i}^{k}}}{{\scriptstyle%
(d_{i}^{k})^{3}}} & \dfrac{{\scriptstyle(\Delta x_{i}^{k})^{2}+(\Delta
z_{i}^{k})^{2}}}{{\scriptstyle(d_{i}^{k})^{3}}} & \dfrac{{\scriptstyle%
-\Delta y_{i}^{k}\Delta z_{i}^{k}}}{{\scriptstyle(d_{i}^{k})^{3}}} \\
\dfrac{{\scriptstyle-\Delta x_{i}^{k}\Delta z_{i}^{k}}}{{\scriptstyle%
(d_{i}^{k})^{3}}} & \dfrac{{\scriptstyle-\Delta y_{i}^{k}\Delta z_{i}^{k}}}{{%
\scriptstyle(d_{i}^{k})^{3}}} & \dfrac{{\scriptstyle(\Delta
x_{i}^{k})^{2}+(\Delta y_{i}^{k})^{2}}}{{\scriptstyle(d_{i}^{k})^{3}}}%
\end{array}%
\right] ,%
\end{array}
\label{eq:U formation}
\end{equation}
and
\begin{equation}
\mathbf{V}_{i}^{k}={\scriptstyle \dfrac{1}{{\scriptstyle {\displaystyle d_{i}^{k}}}}}\left[\begin{array}{c}
\left[{\scriptstyle \left(\dfrac{{\scriptstyle \partial\mathbf{R}_{i\_x}^{\mathrm{T}}}}{{\scriptstyle \partial\theta_{x}}}\right)\mathbf{R}_{i\_y}^{\mathrm{T}}\mathbf{R}_{i\_z}^{\mathrm{T}}\left(\begin{array}{c}
{\Delta x}_{i}^{k}\\
{\Delta y}_{i}^{k}\\
{\Delta z}_{i}^{k}
\end{array}\right)}\right]^{\mathrm{T}}\\
\left[{\scriptstyle {\scriptstyle \mathbf{R}_{i\_x}^{\mathrm{T}}}\left(\dfrac{{\scriptstyle \partial\mathbf{R}_{i\_y}^{\mathrm{T}}}}{{\scriptstyle {\scriptstyle \partial\theta_{y}}}}\right){\scriptstyle \mathbf{R}_{i\_z}^{\mathrm{T}}}{\scriptstyle \left(\begin{array}{c}
{\Delta x}_{i}^{k}\\
{\Delta y}_{i}^{k}\\
{\Delta z}_{i}^{k}
\end{array}\right)}}\right]^{\mathrm{T}}\\
\left[{\scriptstyle \mathbf{R}_{i\_x}^{\mathrm{T}}{\scriptstyle \mathbf{R}_{i\_y}^{\mathrm{T}}\left(\dfrac{{\scriptstyle \partial\mathbf{R}_{i\_z}^{\mathrm{T}}}}{{\scriptstyle \partial\theta_{z}}}\right)}}{\scriptstyle \left(\begin{array}{c}
{\Delta x}_{i}^{k}\\
{\Delta y}_{i}^{k}\\
{\Delta z}_{i}^{k}
\end{array}\right)}\right]^{\mathrm{T}}
\end{array}\right]^{\mathrm{T}}\label{eq:V formation}
\end{equation}
where $\mathbf{R}_{i\_x}, \mathbf{R}_{i\_y}$ and $\mathbf{R}_{i\_z}$
are the rotation  matrices about coordinate frame axes $x, y$, and
$z$, respectively. The expression of $\mathbf{R}_{i}^{\mathrm{T}}$
is as follows:
\begin{equation}\nonumber
\mathbf{R}_{i}^{\mathrm{T}}=\mathbf{R}_{i\_x}^{\mathrm{T}}\mathbf{R}_{i\_y}^{\mathrm{T}}\mathbf{R}_{i\_z}^{\mathrm{T}},
\end{equation}
with
\begin{equation}\nonumber
\begin{array}{c}
\mathbf{R}_{i\_x}=\left[\begin{array}{ccc}
1 & 0 & 0\\
0 & \cos\theta_{x} & -\sin\theta_{x}\\
0 & \sin\theta_{x} & \cos\theta_{x}
\end{array}\right]\\
\mathbf{R}_{i\_y}=\left[\begin{array}{ccc}
\cos\theta_{y} & 0 & \sin\theta_{y}\\
0 & 1 & 0\\
-\sin\theta_{y} & 0 & \cos\theta_{y}
\end{array}\right]\\
\mathbf{R}_{i\_z}=\left[\begin{array}{ccc}
\cos\theta_{z} & -\sin\theta_{z} & 0\\
\sin\theta_{z} & \cos\theta_{z} & 0\\
0 & 0 & 1
\end{array}\right]
\end{array}.
\end{equation}
Denote $\mathbf{T}^{k}=\dfrac{\partial\mathbf{z}^{k}}{\partial\mathbf{s}^{k}}\in\mathbf{\mathbb{R}}^{4\widetilde{N}\times3}$
as the partial derivative of TDOA and DOA measurements w.r.t. sound
source position at time instance $t^{k}$, for $k=1,\ldots,K$. We
then have the expression of $\mathbf{T}^{k}$ as follows:
\begin{equation}
\begin{array}{c}
\mathbf{T}^{k}=\dfrac{\partial\mathbf{z}^{k}}{{\partial}\mathbf{s}^{k}}=\left[\begin{array}{ccc}
\mathbf{J}_{x}^{k} & \mathbf{J}_{y}^{k} & \mathbf{J}_{z}^{k}\end{array}\right]\\
=\left[\begin{array}{c}
-\mathbf{h}_{2}^{k}\\
\mathbf{-U}_{2}^{k}\\
\vdots\\
\mathbf{-h}_{N}^{k}\\
\mathbf{-U}_{N}^{k}
\end{array}\right]-\left[\begin{array}{c}
\left(\dfrac{\mathbf{s}^{k}}{cd_{1}^{k}}\right)^{\mathrm{T}}\\
\mathbf{0}_{3\times3}\\
\vdots\\
\left(\dfrac{\mathbf{s}^{k}}{cd_{1}^{k}}\right)^{\mathrm{T}}\\
\mathbf{0}_{3\times3}
\end{array}\right]
\end{array}.\label{eq:part TK}
\end{equation}
The results then follow the definition of the Jacobian matrix \cite[pp. 569]{Siciliano2009}.
This completes the proof. \end{proof of Proposition }

\begin{proof of Theorem 2}
First, $\mathbf{L}^{k}$ can be expressed as: 
\begin{equation}\nonumber
\mathbf{L}^{k}=diag(\mathbf{H}_{arr\_2}^{k},\mathbf{H}_{arr\_3}^{k},\cdots,\mathbf{H}_{arr\_\widetilde{N}}^{k},\mathbf{H}_{arr\_N}^{k}).
\end{equation}

\noindent By performing elementary row transformation of $\mathbf{F}$, we can obtain: 

\begin{equation}\nonumber
\begin{array}{c}
\overline{\mathbf{F}}=\left[\begin{array}{ccccc}
\mathbf{H}_{arr\_2}^{1} &  &  &  & \mathbf{T}_{arr\_2}^{1}\\
\vdots &  &  &  & \vdots\\
\mathbf{H}_{arr\_2}^{K} &  &  &  & \mathbf{T}_{arr\_2}^{K}\\
 & \mathbf{H}_{arr\_3}^{1} &  &  & \mathbf{T}_{arr\_3}^{1}\\
 & \vdots &  &  & \vdots\\
 & \mathbf{H}_{arr\_3}^{K} &  &  & \mathbf{T}_{arr\_3}^{K}\\
 &  & \ddots &  & \vdots\\
 &  &  & \mathbf{H}_{arr\_N}^{1} & \mathbf{T}_{arr\_N}^{1}\\
 &  &  & \vdots & \vdots\\
 &  &  & \mathbf{H}_{arr\_N}^{K} & \mathbf{T}_{arr\_N}^{K}
\end{array}\right]\\
=\left[\begin{array}{ccccc}
\mathbf{H}_{arr\_2} &  &  &  & \mathbf{T}_{arr\_2}\\
 & \mathbf{H}_{arr\_3} &  &  & \mathbf{T}_{arr\_3}\\
 &  & \ddots &  & \vdots\\
 &  &  & \mathbf{H}_{arr\_N} & \mathbf{T}_{arr\_N}
\end{array}\right]
\end{array}\label{eq:L T}
\end{equation}
where 
\begin{equation}\nonumber
\begin{array}{c}
\mathbf{H}_{arr\_i}=\left[\begin{array}{ccc}
\mathbf{H}_{arr\_i}^{1}; \cdots; \mathbf{H}_{arr\_i}^{K}\end{array}\right]\in\mathbf{\mathbb{R}}^{4K\times8}\\
\mathbf{T}_{arr\_i}=\left[\begin{array}{ccc}
\mathbf{T}_{arr\_i}^{1}; \cdots; \mathbf{T}_{arr\_i}^{K}\end{array}\right]\in\mathbf{\mathbb{R}}^{4K\times3}
\end{array}
\end{equation}
for $i=2,...,N$. Apparently, it holds that $rank(\mathbf{F})=rank(\overline{\mathbf{F}})$.
Also, due to the structure of $\mathbf{H}_{arr\_i}$, their columns
are independent of each other. For each microphone array, denote $\mathbf{F}_{arr\_i}=\left[\begin{array}{cc}
\mathbf{H}_{arr\_i} & \mathbf{T}_{arr\_i}\end{array}\right]$. We then perform the following elementary transformation on the matrix $\mathbf{F}_{arr\_i}$: 

(i) adding the first column block $\left[\mathbf{h}_{i}^{1};\mathbf{U}_{i}^{1};\cdots;\mathbf{h}_{i}^{K};\mathbf{U}_{i}^{K}\right]$ of $\mathbf{H}_{arr\_i}$ to $\mathbf{T}_{arr\_i}$;

(ii) exchanging row blocks to collect all $\mathbf{h}_{i}^{k}$ and
$\mathbf{U}_{i}^{k}$ together, respectively, thereby obtaining 
\begin{equation}\nonumber
\overline{\mathbf{F}}_{arr\_i}={\left[\begin{array}{cccc}
{\scriptstyle \mathbf{M}_{h\_i}} & {\scriptstyle \mathbf{0}_{K\times3}} & {\scriptstyle \mathbf{1}_{K\times1}} & {\scriptstyle \mathbf{{k}} \Delta }_{t}\\
{\scriptstyle \mathbf{M}_{U\_i}} & {\scriptstyle \mathbf{M}_{V\_i}} & {\scriptstyle \mathbf{0}_{3K\times1}} & {\scriptstyle \mathbf{0}_{3K\times1}}
\end{array}\right.}{\left.\begin{array}{c}
{\scriptstyle -\mathbf{t}_{K}}\\
{\scriptstyle \mathbf{0}_{3K\times1}}
\end{array}\right]}\in\mathbf{\mathbb{R}}^{4K\times11}\label{eq:L';T'}
\end{equation}
where
\begin{equation}\nonumber
\begin{cases}
\mathbf{k=}\left[\begin{array}{ccc}
1; 2; \ldots; K\end{array}\right], \text{ } \mathbf{M}_{h\_i}=[\mathbf{h}_{i}^{1};\mathbf{h}_{i}^{2};\ldots;\mathbf{h}_{i}^{K}],\\
\mathbf{M}_{U\_i}=\left[\mathbf{U}_{i}^{1};\mathbf{U}_{i}^{2};\ldots;\mathbf{U}_{i}^{K}\right],\text{ }
\mathbf{M}_{V\_i}=\left[\mathbf{V}_{i}^{1};\mathbf{V}_{i}^{2};\ldots;\mathbf{V}_{i}^{K}\right]\text{,}\\
\mathbf{t}_{K}=\left[\begin{array}{c}
\left(\frac{{\scriptstyle \mathbf{s}^{1}}}{{\scriptstyle cd_{1}^{k}}}\right)^{\mathrm{T}};\left(\frac{{\scriptstyle \mathbf{s}^{2}}}{{\scriptstyle cd_{1}^{k}}}\right)^{\mathrm{T}};\left(\frac{{\scriptstyle \mathbf{s}^{K}}}{{\scriptstyle cd_{1}^{k}}}\right)^{\mathrm{T}}\end{array}\right].
\end{cases}
\end{equation}

We further perform the following elementary operations on
$\overline{\mathbf{F}}_{arr\_i}$, $i=2,3,\cdots,N$:

(i) dividing the fourth column block by ${\Delta}_{t}$;

(ii) for $k=2,3,\cdots,K$, deducing the  $k\raisebox{0mm}{-}th$ row by the first row;

(iii) transforming the elements in the first row (except the third one) to zero by the third column block (the first element therein equals 1 while the other elements equal zero after the elementary operations listed above);

(iv) for $k=3,4,\cdots,K$, deducing the   $k\raisebox{0mm}{-}th$ row by the second row multiplied by $(k-1)$;

(v) transforming the elements in the second row (except the fourth one) to zero by the fourth column block (the second element therein equals 1 while the other elements equal zero after the elementary operations listed above);

(vi) moving column blocks 3 and 4 to columns blocks 1 and 2, respectively.

After the above operations, we obtain
\begin{equation}\nonumber
\begin{array}{l}
\overline{\mathbf{F}}_{arr\_i}^{\prime}=\left[\begin{array}{cc}
\mathbf{\mathbf{\bar{L}}_{i}} & \bar{\mathbf{T}}\end{array}\right]\end{array}\label{eq:F'_arr_i}
\end{equation}
where $\mathbf{\mathbf{\bar{L}}_{i}}$ and $\bar{\mathbf{T}}$ are
shown in (\ref{eq: L-BAR}) and (\ref{eq:T-BAR}), respectively. With the above elementary transformations, we
have 
\begin{equation}
\overline{\mathbf{F}}\sim\overline{\mathbf{F}}^{\prime}=\left[\begin{array}{ccccc}
\mathbf{\mathbf{\bar{L}}_{2}} &  &  &  & \mathbf{\bar{T}}\\
 & \mathbf{\mathbf{\bar{L}}_{3}} &  &  & \mathbf{\bar{T}}\\
 &  & \ddots &  & \mathbf{\vdots}\\
 &  &  & \mathbf{\mathbf{\bar{L}}_{N}} & \mathbf{\bar{T}}
\end{array}\right].\label{eq:F bar prime}
\end{equation}
It holds that $rank(\mathbf{F})=rank(\overline{\mathbf{F}})=rank(\overline{\mathbf{F}}^{\prime})$.
From the structure of $\overline{\mathbf{F}}^{\prime}$, we can see
that the block columns containing $\mathbf{\mathbf{\bar{L}}_{i}}$,
$i=2,...,N$, are independent of each other. A necessary condition
for $\overline{\mathbf{F}}^{\prime}$ to be of full  column rank is
that $\mathbf{\mathbf{\bar{L}}_{i}}$\ and $\mathbf{\bar{T}}$ are
of full column rank, respectively, $i=2,...,N$. This completes the proof.\end{proof of Theorem 2}

\begin{proof of Theorem 3}
 Here we take $j=2$ as an example. For $\overline{\mathbf{F}}^{\prime}$, we could perform elementary row block changes: for $i=3,\ldots,N$,
deduce $\mathbf{\bar{L}_{i}}$ row block by the first-row block and
obtain:
\begin{equation}
\left[\begin{array}{cccccc}
\mathbf{\bar{L}}_{2} &  &  &  &  & \mathbf{\bar{T}}\\
-\mathbf{\bar{L}}_{2} & \mathbf{\mathbf{\bar{L}}_{3}} &  &  &  & \mathbf{0}\\
\vdots &  & \ddots &  &  & \mathbf{\vdots}\\
-\mathbf{\bar{L}}_{2} &  &  &  & \mathbf{\mathbf{\bar{L}}_{N}} & \mathbf{0}
\end{array}\right].\label{eq:L_sufficient}
\end{equation}

 \noindent Denote the submatrix of this matrix as:
\begin{equation}\nonumber
\mathbf{M}_{2\_T}=\left[\begin{array}{cc}
\mathbf{\mathbf{\bar{L}}_{2}} & \mathbf{\bar{T}}\\
\mathbf{\vdots} & \mathbf{\vdots}\\
-\mathbf{\mathbf{\bar{L}}_{2}} & \mathbf{0}
\end{array}\right].\label{eq:M2_t}
\end{equation}

\noindent From the structure in (\ref{eq:L_sufficient}), we can see clearly that if:

(i) $\mathbf{M}_{2\_T}$ is of full column rank, and

(ii) $diag(\mathbf{\mathbf{\bar{L}}_{3},\ldots,\mathbf{\bar{L}}_{N}})$\ is
of full column rank,\\
then $\overline{\mathbf{F}}^{\prime}$ will be of full column rank. Due to the fact that $rank(\mathbf{F})=rank(\overline{\mathbf{F}})=rank(\overline{\mathbf{F}}^{\prime})$,
the Jacobian matrix $\mathbf{J}$ is of full column rank. Similarly, the same conditions hold when $j$ equals to $3,\ldots,N$. So
the Jacobian matrix $\mathbf{J}$ is of full column rank if any matrix
consisting of the $(j-1)\raisebox{0mm}{-}th$ column block and the last column block
in $\overline{\mathbf{F}}^{\prime}$ is of full column rank, $2\leq j\leq N$,
and $\mathbf{\mathbf{\bar{L}}_{i}}$ are of full column rank, $i=2,\ldots,N$
and $i\neq j$. This completes the proof. \end{proof of Theorem 3}

\begin{proof of Theorem 4}
(i) $\bar{\mathbf{T}}$ in (\ref{eq:T-BAR}) is of full column rank
only if a 3 \texttimes{} 3 matrix formed by at least one of the three-permutation
of its rows is full rank. For $\left(\mathbf{s}^{k}\right)^{\mathrm{T}}\in R^{1\times3}, 1\leq k\leq K$,
the necessary condition for $\bar{\mathbf{T}}$ to be of full column
rank is $K\geq5$. If $K<5$, $\bar{\mathbf{T}}$ can not be of the
full column rank.

(ii) Based on (\ref{eq:s}), when $\mathbf{\mathbf{s}}^{k}=\mathbf{\mathbf{{\scriptstyle \lambda}s}}^{k-1}$, we could derive $\frac{\mathbf{s}^{k}}{d_{1}^{k}}=\frac{\mathbf{s}^{k-1}}{d_{1}^{k-1}}$. From the expression of $\bar{\mathbf{T}}$, we can see that $\bar{\mathbf{T}}$ cannot
be of full rank if $\mathbf{s}^{k}$ is proportional to each other, $k=1,\cdots,K$.
In this situation, the sound
source positions at all time steps are collinear (together with the origin) w.r.t. the reference
microphone array frame.

(iii) If the sound source keeps moving in any planes of $x=\alpha y$,
$x=\beta z$, $y=\gamma z$ w.r.t. $\left\{ \mathrm{\mathbf{x}}_{arr\_1}\right\} $
at all moments, where $\alpha,\beta$, and $\gamma$ are arbitrary real
numbers, the sound source position $\mathbf{s}^{k},$ $1\leq k\leq K$,
could be expressed as $\left[\alpha s_{y}^{k};s_{y}^{k};s_{z}^{k}\right]$,
$\left[\beta s_{z}^{k};s_{y}^{k};s_{z}^{k}\right]$, and  $\left[s_{x}^{k};\gamma s_{z}^{k};s_{z}^{k}\right]$,
respectively. $\bar{\mathbf{T}}$ will not be of full column rank.

Specifically, if $\alpha=0$ or $\beta=0$ or $\gamma=0$, the
sound source position of $\mathbf{s}^{k}$ will have $s_{x}^{k}=0$,
$s_{y}^{k}=0$, and $s_{z}^{k}=0$, respectively, i.e., YOZ, XOZ, and
XOY planes. If the sound source keeps moving in the line of $x=\alpha y=\beta z$,
the situation will change to (ii). This completes the proof. \end{proof of Theorem 4}

\begin{proof of Theorem 5}
(i) If the sound source positions w.r.t. $\left\{ \mathrm{\mathbf{x}}_{arr\_i}\right\} $
at all of $K\,(K\geq5)$ time steps are collinear, i.e., $(\mathbf{\mathbf{s}}^{k}-\mathbf{x}_{arr\_i}^{p})=\lambda(\mathbf{\mathbf{s}}^{k-1}-\mathbf{x}_{arr\_i}^{p})$
is always true. For $i\geq2,$ $k=2,3,\ldots,K$, we can get the following expression:
\begin{equation}\nonumber
\begin{cases}
{\scriptstyle \left[\begin{array}{ccc}
\Delta x_{i}^{k}; \Delta y_{i}^{k}; \Delta z_{i}^{k}\end{array}\right]=\lambda\left[\begin{array}{ccc}
\Delta x^{k-1}; \Delta y^{k-1}; \Delta z^{k-1}\end{array}\right]},\\
\mathbf{h}_{i}^{k}=\mathbf{h}_{i}^{k-1}, \text{ } \mathbf{U}_{i}^{k}=\frac{1}{\lambda}\mathbf{U}_{i}^{k-1}, \text{ } \mathbf{V}_{i}^{k}=\mathbf{V}_{i}^{k-1}.\\
\end{cases}
\end{equation}
where $\mathbf{h},\mathbf{U}$, and $\mathbf{V}$ are defined in (\ref{eq:Block of neccessary matrix}). 

For an arbitrary single time step, we have $rank(\mathbf{U}_{i}^{k})=rank(\mathbf{R}_{i}^{\mathrm{T}}\mathbf{A})$
as shown in (\ref{eq:U formation}). It can also be seen that $det(\mathbf{A})=0$ and the
second-order sub-determinant of $\mathbf{A}$ is not equal to 0, we
know that $rank(\mathbf{A})=2$. $\mathbf{R}_{i}^{\mathrm{T}}$ is
a rotation matrix, $rank(\mathbf{R}_{i}^{\mathrm{T}})=3$, thus $rank(\mathbf{U}_{i}^{k})=2$.
Therefore, $\mathbf{\bar{L}}_{i}$ will not be of full column rank.

(ii) When $\theta_{arr\_i}^{y}=\frac{\pi}{2}$, for the corresponding microphone
array at any different time steps, $\mathbf{V}_{i}^{k}$ defined
in (\ref{eq:V formation}) has the same structure, i.e., 
\begin{equation}
\mathbf{V}_{i}^{k}=\left[
\begin{array}{cc}
{\scriptstyle0} & {\scriptstyle\Delta x_{i}^{k}c_{z}+\Delta y_{i}^{k}s_{z}}
\\
{\scriptstyle\Delta y_{i}^{k}s_{x-z}-\Delta x_{i}^{k}c_{x-z}} & {\scriptstyle%
\Delta z_{i}^{k}s_{x}} \\
{\scriptstyle\Delta y_{i}^{k}c_{x-z}+\Delta x_{i}^{k}s_{x-z}} & {\scriptstyle%
\Delta z_{i}^{k}c_{x}}%
\end{array}%
\right. \left.
\begin{array}{c}
{\scriptstyle0} \\
{\scriptstyle-\Delta y_{i}^{k}s_{x-z}+\Delta x_{i}^{k}c_{x-z}} \\
{\scriptstyle-\Delta y_{i}^{k}c_{x-z}-\Delta x_{i}^{k}s_{x-z}}%
\end{array}%
\right] ,  \nonumber
\end{equation}
where $s,c$ represent $sin,cos$, respectively and $rank(\mathbf{V}_{i}^{k})\equiv2$.
Therefore, the matrix of $\mathbf{\bar{L}}_{i}$ in (\ref{eq:F bar prime})
will not be of full column rank. This completes the proof.
\end{proof of Theorem 5}

\end{appendix}

\end{document}